\documentclass{article}

\usepackage{times}
\usepackage{graphicx} % more modern
\usepackage{subfigure}

\usepackage{natbib}

\usepackage{algorithm}
\usepackage{algorithmic}
\usepackage[algo2e,linesnumbered,lined,boxed,vlined]{algorithm2e}

\usepackage{amsfonts}
\usepackage{amssymb}

\usepackage{hyperref}

\usepackage[accepted]{icml2016}

\usepackage{comment}
\usepackage{amsthm}
\usepackage{mathtools}
\usepackage{tabularx}
\usepackage{multirow}

\newtheorem{proposition}{Proposition}
\newtheorem{corollary}{Corollary}

\def\b{\ensuremath\boldsymbol}

\usepackage{lastpage}
\usepackage{fancyhdr}
\usepackage{url}
\icmltitlerunning{Multidimensional Scaling, Sammon Mapping, and Isomap: Tutorial and Survey}

\begin{document}

\AddToShipoutPictureBG*{%
  \AtPageUpperLeft{%
    \setlength\unitlength{1in}%
    \hspace*{\dimexpr0.5\paperwidth\relax}
    \makebox(0,-0.75)[c]{\normalsize {\color{black} To appear as a part of an upcoming academic book on dimensionality reduction and manifold learning.}}
    }}

\twocolumn[
\icmltitle{Multidimensional Scaling, Sammon Mapping, and Isomap: \\Tutorial and Survey}

\icmlauthor{Benyamin Ghojogh}{bghojogh@uwaterloo.ca}
\icmladdress{Department of Electrical and Computer Engineering, 
\\Machine Learning Laboratory, University of Waterloo, Waterloo, ON, Canada}
\icmlauthor{Ali Ghodsi}{ali.ghodsi@uwaterloo.ca}
\icmladdress{Department of Statistics and Actuarial Science \& David R. Cheriton School of Computer Science, 
\\Data Analytics Laboratory, University of Waterloo, Waterloo, ON, Canada}
\icmlauthor{Fakhri Karray}{karray@uwaterloo.ca}
\icmladdress{Department of Electrical and Computer Engineering, 
\\Centre for Pattern Analysis and Machine Intelligence, University of Waterloo, Waterloo, ON, Canada}
\icmlauthor{Mark Crowley}{mcrowley@uwaterloo.ca}
\icmladdress{Department of Electrical and Computer Engineering, 
\\Machine Learning Laboratory, University of Waterloo, Waterloo, ON, Canada}

\icmlkeywords{Tutorial}

\vskip 0.3in
]

\begin{abstract}
Multidimensional Scaling (MDS) is one of the first fundamental manifold learning methods. It can be categorized into several methods, i.e., classical MDS, kernel classical MDS, metric MDS, and non-metric MDS. Sammon mapping and Isomap can be considered as special cases of metric MDS and kernel classical MDS, respectively. In this tutorial and survey paper, we review the theory of MDS, Sammon mapping, and Isomap in detail. We explain all the mentioned categories of MDS. Then, Sammon mapping, Isomap, and kernel Isomap are explained. Out-of-sample embedding for MDS and Isomap using eigenfunctions and kernel mapping are introduced. Then, Nystrom approximation and its use in landmark MDS and landmark Isomap are introduced for big data embedding. We also provide some simulations for illustrating the embedding by these methods. 
\end{abstract}

\section{Introduction}

Multidimensional Scaling (MDS) \cite{cox2008multidimensional}, first proposed in \cite{torgerson1952multidimensional}, is one of the earliest proposed manifold learning methods. 
It can be used for manifold learning, dimensionality reduction, and feature extraction \cite{ghojogh2019feature}. 
The idea of MDS is to preserve the similarity \cite{torgerson1965multidimensional} or dissimilarity/distances \cite{beals1968foundations} of points in the low-dimensional embedding space. 
Hence, it fits the data locally to capture the global structure of data \cite{saul2003think}.
MDS can be categorized into classical MDS, metric MDS, and non-metric MDS.

In later approaches, Sammon mapping \cite{sammon1969nonlinear} was proposed which is a special case of the distance-based metric MDS. One can consider Sammon mapping as the first proposed nonlinear manifold learning method \cite{ghojogh2019roweis}. The disadvantage of Sammon mapping is its iterative solution of optimization, which makes this method a little slow.

The classical MDS can be generalized to have kernel classical MDS in which any valid kernel can be used. Isomap \cite{tenenbaum2000global} is a special case of the kernel classical MDS which uses a kernel constructed from geodesic distances between points. Because of the nonlinearity of geodesic distance, Isomap is also a nonlinear manifold learning method. 

MDS and its special cases, Sammon mapping, and Isomap have had different applications \cite{young2013multidimensional}. For example, MDS has been used for facial expression recognition \cite{russell1985multidimensional,katsikitis1997classification}. Kernel Isomap has also been used for this application \cite{zhao2011facial}. 

The goal is to embed the high-dimensional input data $\{\b{x}_i\}_{i=1}^n$ into the lower dimensional embedded data $\{\b{y}_i\}_{i=1}^n$ where $n$ is the number of data points. We denote the dimensionality of input and embedding spaces by $d$ and $p \leq d$, respectively, i.e. $\b{x}_i \in \mathbb{R}^d$ and $\b{y}_i \in \mathbb{R}^p$. We denote $\mathbb{R}^{d \times n} \ni \b{X} := [\b{x}_1, \dots, \b{x}_n]$ and $\mathbb{R}^{p \times n} \ni \b{Y} := [\b{y}_1, \dots, \b{y}_n]$. 

The remainder of this paper is organized as follows. Section \ref{section_MDS} explains MDS and its different categories, i.e., classical MDS, generalized classical MDS (kernel classical MDS), metric MDS, and non-metric MDS. Sammon mapping and Isomap are introduced in Sections \ref{section_Sammon_mapping} and \ref{section_Isomap}, respectively. Section \ref{section_outOfSample} introduced the methods for out-of-sample extensions of MDS and Isomap methods. Landmark MDS and landmark Isomap, for big data embedding, are explained in Section \ref{section_landmark_methods}. Some simulations for illustrating the results of embedding are provided in Section \ref{section_simulations}. Finally, Section \ref{section_conclusion} concludes the paper. 

\section{Multidimensional Scaling}\label{section_MDS}

MDS, first proposed in \cite{torgerson1952multidimensional}, can be divided into several different categories \cite{cox2008multidimensional,borg2005modern}, i.e., classical MDS, metric MDS, and non-metric MDS. Note that the results of these are different \cite{jung2013lecture}.
In the following, we explain all three categories. 

\subsection{Classical Multidimensional Scaling}

\subsubsection{Classical MDS with Euclidean Distance}

The \textit{classical MDS} is also referred to as \textit{Principal Coordinates Analysis (PCoA)}, or \textit{Torgerson Scaling}, or \textit{Torgerson–Gower scaling} \cite{gower1966some}. 
The goal of classical MDS is to preserve the similarity of data points in the embedding space as it was in the input space \cite{torgerson1965multidimensional}. One way to measure similarity is inner product. Hence, we can minimize the difference of similarities in the input and embedding spaces: 
\begin{equation}
\begin{aligned}
& \underset{\{\b{y}_i\}_{i=1}^n}{\text{minimize}}
& & c_1 := \sum_{i=1}^n \sum_{j=1}^n (\b{x}_i^\top \b{x}_j - \b{y}_i^\top \b{y}_j)^2,
\end{aligned}
\end{equation}
whose matrix form is:
\begin{equation}\label{equation_MDS_optimization_matrix}
\begin{aligned}
& \underset{\b{Y}}{\text{minimize}}
& & c_1 = ||\b{X}^\top \b{X} - \b{Y}^\top \b{Y}||_F^2,
\end{aligned}
\end{equation}
where $\|\!\cdot\!\|_F$ denotes the Frobenius norm, and $\b{X}^\top \b{X}$ and $\b{Y}^\top \b{Y}$ are the Gram matrices of the original data $\b{X}$ and the embedded data $\b{Y}$, respectively.

The objective function, in Eq. (\ref{equation_MDS_optimization_matrix}), is simplified as:
\begin{align*}
||\b{X}^\top \b{X} &- \b{Y}^\top \b{Y}||_F^2 \\
&= \textbf{tr}\big[(\b{X}^\top \b{X} - \b{Y}^\top \b{Y})^\top (\b{X}^\top \b{X} - \b{Y}^\top \b{Y})\big] \\
&= \textbf{tr}\big[(\b{X}^\top \b{X} - \b{Y}^\top \b{Y}) (\b{X}^\top \b{X} - \b{Y}^\top \b{Y})\big] \\
&= \textbf{tr}\big[(\b{X}^\top \b{X} - \b{Y}^\top \b{Y})^2\big],
\end{align*}
where $\textbf{tr}(.)$ denotes the trace of matrix.
If we decompose $\b{X}^\top \b{X}$ and $\b{Y}^\top \b{Y}$ using eigenvalue decomposition \cite{ghojogh2019eigenvalue}, we have:
\begin{align}
& \b{X}^\top \b{X} = \b{V} \b{\Delta} \b{V}^\top, \label{equation_classical_MDS_X_decompose} \\
& \b{Y}^\top \b{Y} = \b{Q} \b{\Psi} \b{Q}^\top, \label{equation_classical_MDS_Y_decompose}
\end{align}
where eigenvectors are sorted from leading (largest eigenvalue) to trailing (smallest eigenvalue). 
Note that, rather than eigenvalue decomposition of $\b{X}^\top \b{X}$ and $\b{Y}^\top \b{Y}$, one can decompose $\b{X}$ and $\b{Y}$ using Singular Value Decomposition (SVD) and take the right singular vectors of $\b{X}$ and $\b{Y}$ as $\b{V}$ and $\b{Q}$, respectively. The matrices $\b{\Delta}$ and $\b{\Psi}$ are the obtained by squaring the singular values (to power $2$). See {\citep[Proposition 1]{ghojogh2019unsupervised}} for proof. 

% Note that, rather than eigenvalue decomposition, one can decompose $\b{X}^\top \b{X}$ and $\b{Y}^\top \b{Y}$, as in Eqs. (\ref{equation_classical_MDS_X_decompose}) and (\ref{equation_classical_MDS_Y_decompose}), using Singular Value Decomposition (SVD), where for both $\b{X}^\top \b{X}$ and $\b{Y}^\top \b{Y}$, the left and right singular matrices are equivalent because $\b{X}^\top \b{X}$ and $\b{Y}^\top \b{Y}$ are symmetric.

The objective function can be further simplified as:
\begin{align*}
\therefore ~~~ &||\b{X}^\top \b{X} - \b{Y}^\top \b{Y}||_F^2 \\
&= \textbf{tr}\big[(\b{X}^\top \b{X} - \b{Y}^\top \b{Y})^2\big] \\ 
&= \textbf{tr}\big[(\b{V} \b{\Delta} \b{V}^\top - \b{Q} \b{\Psi} \b{Q}^\top)^2\big] \\
&\overset{(a)}{=} \textbf{tr}\big[(\b{V} \b{\Delta} \b{V}^\top - \b{V}\b{V}^\top\b{Q} \b{\Psi} \b{Q}^\top\b{V}\b{V}^\top)^2\big] \\
&= \textbf{tr}\Big[\big(\b{V} (\b{\Delta} - \b{V}^\top\b{Q} \b{\Psi} \b{Q}^\top\b{V}) \b{V}^\top \big)^2\Big] \\
&= \textbf{tr}\Big[\b{V}^2 (\b{\Delta} - \b{V}^\top\b{Q} \b{\Psi} \b{Q}^\top\b{V})^2 (\b{V}^\top)^2\Big] \\
&\overset{(b)}{=} \textbf{tr}\Big[(\b{V}^\top)^2 \b{V}^2 (\b{\Delta} - \b{V}^\top\b{Q} \b{\Psi} \b{Q}^\top\b{V})^2\Big] \\
&= \textbf{tr}\Big[(\underbrace{\b{V}^\top \b{V}}_{\b{I}})^2 (\b{\Delta} - \b{V}^\top\b{Q} \b{\Psi} \b{Q}^\top\b{V})^2\Big] \\
&\overset{(c)}{=} \textbf{tr}\Big[(\b{\Delta} - \b{V}^\top\b{Q} \b{\Psi} \b{Q}^\top\b{V})^2\Big],
\end{align*}
where $(a)$ and $(c)$ are for $\b{V}^\top \b{V} = \b{V} \b{V}^\top = \b{I}$ because $\b{V}$ is a non-truncated (square) orthogonal matrix (where $\b{I}$ denotes the identity matrix). The reason of $(b)$ is the cyclic property of trace.

Let $\mathbb{R}^{n \times n} \ni \b{M} := \b{V}^\top \b{Q}$, so:
\begin{align*}
& ||\b{X}^\top \b{X} - \b{Y}^\top \b{Y}||_F^2 = \textbf{tr}\Big[(\b{\Delta} - \b{M} \b{\Psi} \b{M}^\top)^2\Big].
\end{align*}
Therefore:
\begin{align*}
\therefore ~~~~~ &\underset{\b{Y}}{\text{minimize}}~~ ||\b{X}^\top \b{X} - \b{Y}^\top \b{Y}||_F^2 \\
&\equiv \underset{\b{M}, \b{\Psi}}{\text{minimize}}~~ \textbf{tr}\Big[(\b{\Delta} - \b{M} \b{\Psi} \b{M}^\top)^2\Big].
\end{align*}
The objective function is:
\begin{align*}
c_1 &= \textbf{tr}\Big[(\b{\Delta} - \b{M} \b{\Psi} \b{M}^\top)^2\Big] \\
&= \textbf{tr}(\b{\Delta}^2 + (\b{M} \b{\Psi} \b{M}^\top)^2 - 2\b{\Delta} \b{M} \b{\Psi} \b{M}^\top) \\
&= \textbf{tr}(\b{\Delta}^2) + \textbf{tr}((\b{M} \b{\Psi} \b{M}^\top)^2) - 2\,\textbf{tr}(\b{\Delta} \b{M} \b{\Psi} \b{M}^\top).
\end{align*}
As the optimization problem is unconstrained and the objective function is the trace of a quadratic function, the minimum is non-negative. 

If we take derivative with respect to the first objective variable, i.e., $\b{M}$, we have:
\begin{align}
&\mathbb{R}^{n \times n} \ni \frac{\partial c_1}{\partial \b{M}} = 2(\b{M} \b{\Psi} \b{M}^\top) \b{M} \b{\Psi} - 2\b{\Delta} \b{M} \b{\Psi} \overset{\text{set}}{=} \b{0} \nonumber \\
&\implies (\b{M} \b{\Psi} \b{M}^\top) (\b{M} \b{\Psi}) = (\b{\Delta}) (\b{M} \b{\Psi}) \nonumber \\
&\overset{(a)}{\implies} \b{M} \b{\Psi} \b{M}^\top = \b{\Delta}, \label{equation_metric_MDS_Delta_1}
\end{align}
where $(a)$ is because $\b{M} \b{\Psi} \neq \b{0}$.

For the derivative with respect to the second objective variable, i.e., $\b{\Psi}$, we simplify the objective function a little bit:
\begin{align*}
c_1 &= \textbf{tr}(\b{\Delta}^2) + \textbf{tr}((\b{M} \b{\Psi} \b{M}^\top)^2) - 2\,\textbf{tr}(\b{\Delta} \b{M} \b{\Psi} \b{M}^\top) \\
&= \textbf{tr}(\b{\Delta}^2) + \textbf{tr}(\b{M}^2 \b{\Psi}^2 \b{M}^{\top 2}) - 2\,\textbf{tr}(\b{\Delta} \b{M} \b{\Psi} \b{M}^\top) \\
&\overset{(a)}{=} \textbf{tr}(\b{\Delta}^2) + \textbf{tr}(\b{M}^{\top 2} \b{M}^2 \b{\Psi}^2) - 2\,\textbf{tr}(\b{M}^\top \b{\Delta} \b{M} \b{\Psi}) \\
&= \textbf{tr}(\b{\Delta}^2) + \textbf{tr}((\b{M}^{\top} \b{M} \b{\Psi})^2) - 2\,\textbf{tr}(\b{M}^\top \b{\Delta} \b{M} \b{\Psi}),
\end{align*}
where $(a)$ is because of the cyclic property of trace.

Taking derivative with respect to the second objective variable, i.e., $\b{\Psi}$, gives:
\begin{align}
&\mathbb{R}^{n \times n} \ni \frac{\partial c_1}{\partial \b{\Psi}} = 2 \b{M}^\top (\b{M} \b{\Psi} \b{M}^\top) \b{M} - 2 \b{M}^\top \b{\Delta} \b{M} \overset{\text{set}}{=} \b{0} \nonumber \\
&\implies \b{M}^\top (\b{M} \b{\Psi} \b{M}^\top) \b{M} =  \b{M}^\top (\b{\Delta}) \b{M} \nonumber \\
&\overset{(a)}{\implies} \b{M} \b{\Psi} \b{M}^\top = \b{\Delta}, \label{equation_metric_MDS_Delta_2}
\end{align}
where $(a)$ is because $\b{M} \neq \b{0}$.
Both Eqs. (\ref{equation_metric_MDS_Delta_1}) and (\ref{equation_metric_MDS_Delta_2}) are:
\begin{align*}
\b{M} \b{\Psi} \b{M}^\top = \b{\Delta},
\end{align*}
whose one possible solution is:
\begin{align}
&\b{M} = \b{I}, \label{equation_metric_MDS_M_I} \\
&\b{\Psi} = \b{\Delta}. \label{equation_metric_MDS_Psi}
\end{align}
which means that the minimum value of the non-negative objective function $\textbf{tr}((\b{\Delta} - \b{M} \b{\Psi} \b{M}^\top)^2)$ is zero.

We had $\b{M} = \b{V}^\top \b{Q}$. Therefore, according to Eq. (\ref{equation_metric_MDS_M_I}), we have:
\begin{align}\label{equation_metric_MDS_Q}
\therefore ~~~ \b{V}^\top \b{Q} = \b{I} \implies \b{Q} = \b{V}.
\end{align}

According to Eq. (\ref{equation_classical_MDS_Y_decompose}), we have:
\begin{align}
\b{Y}^\top \b{Y} = \b{Q} \b{\Psi} \b{Q}^\top &\overset{(a)}{=} \b{Q} \b{\Psi}^{\frac{1}{2}} \b{\Psi}^{\frac{1}{2}} \b{Q}^\top \implies \b{Y} = \b{\Psi}^{\frac{1}{2}} \b{Q}^\top \nonumber \\
&\overset{(\ref{equation_metric_MDS_Psi}),(\ref{equation_metric_MDS_Q})}{\implies} \b{Y} = \b{\Delta}^{\frac{1}{2}} \b{V}^\top, \label{equation_metric_MDS_Y}
\end{align}
where $(a)$ can be done because $\b{\Psi}$ does not include negative entry as the gram matrix $\b{Y}^\top \b{Y}$ is positive semi-definite by definition. 

In summary, for embedding $\b{X}$ using classical MDS, the eigenvalue decomposition of $\b{X}^\top \b{X}$ is obtained as in Eq. (\ref{equation_classical_MDS_X_decompose}). Then, using Eq. (\ref{equation_metric_MDS_Y}), $\b{Y} \in \mathbb{R}^{n \times n}$ is obtained. Truncating this $\b{Y}$ to have $\b{Y} \in \mathbb{R}^{p \times n}$, with the first (top) $p$ rows, gives us the $p$-dimensional embedding of the $n$ points. Note that the leading $p$ columns are used because singular values are sorted from largest to smallest in SVD which can be used for Eq. (\ref{equation_classical_MDS_X_decompose}). 

\subsubsection{Generalized Classical MDS (Kernel Classical MDS)}\label{section_generalized_classical_MDS}

If $d_{ij}^2 = ||\b{x}_i - \b{x}_j||_2^2$ is the squared Euclidean distance between $\b{x}_i$ and $\b{x}_j$, we have:
\begin{align*}
d_{ij}^2 &= ||\b{x}_i - \b{x}_j||_2^2 = (\b{x}_i - \b{x}_j)^\top (\b{x}_i - \b{x}_j) \\
&= \b{x}_i^\top \b{x}_i - \b{x}_i^\top \b{x}_j - \b{x}_j^\top \b{x}_i + \b{x}_j^\top \b{x}_j \\
&= \b{x}_i^\top \b{x}_i - 2\b{x}_i^\top \b{x}_j + \b{x}_j^\top \b{x}_j = \b{G}_{ii} - 2 \b{G}_{ij} + \b{G}_{jj},
\end{align*}
where $\mathbb{R}^{n \times n} \ni \b{G} := \b{X}^\top \b{X}$ is the Gram matrix. If $\mathbb{R}^n \ni \b{g} := [\b{g}_1, \dots, \b{g}_n] = [\b{G}_{11}, \dots, \b{G}_{nn}] = \textbf{diag}(\b{G})$, we have:
\begin{align*}
& d_{ij}^2 = \b{g}_i -2\b{G}_{ij} + \b{g}_j, \\
& \b{D} = \b{g}\b{1}^\top -2 \b{G} +\b{1}\b{g}^\top = \b{1}\b{g}^\top -2 \b{G} + \b{g}\b{1}^\top,
\end{align*}
where $\b{1}$ is the vector of ones and $\b{D}$ is the distance matrix with squared Euclidean distance ($d_{ij}^2$ as its elements). 
Let $\mathbb{R}^{n \times n} \ni \b{H} := \b{I} - \frac{1}{n}\b{1}\b{1}^\top$ denote the centering matrix. 
We double-center the matrix $\b{D}$ as follows \cite{oldford2018lecture}:
\begin{align*}
\b{HDH} &= (\b{I} - \frac{1}{n}\b{1}\b{1}^\top) \b{D} (\b{I} - \frac{1}{n}\b{1}\b{1}^\top) \\
&= (\b{I} - \frac{1}{n}\b{1}\b{1}^\top) (\b{1}\b{g}^\top -2 \b{G} + \b{g}\b{1}^\top) (\b{I} - \frac{1}{n}\b{1}\b{1}^\top) \\
&= \big[\underbrace{(\b{I} - \frac{1}{n}\b{1}\b{1}^\top)\b{1}}_{=\,\b{0}} \b{g}^\top -2 (\b{I} - \frac{1}{n}\b{1}\b{1}^\top)\b{G}\\ 
&~~~~~ + (\b{I} - \frac{1}{n}\b{1}\b{1}^\top)\b{g}\b{1}^\top\big] (\b{I} - \frac{1}{n}\b{1}\b{1}^\top) \\
&= -2 (\b{I} - \frac{1}{n}\b{1}\b{1}^\top)\b{G}(\b{I} - \frac{1}{n}\b{1}\b{1}^\top) \\
&~~~~~ + (\b{I} - \frac{1}{n}\b{1}\b{1}^\top)\b{g}\underbrace{\b{1}^\top(\b{I} - \frac{1}{n}\b{1}\b{1}^\top)}_{=\,\b{0}} \\
&= -2 (\b{I} - \frac{1}{n}\b{1}\b{1}^\top)\b{G}(\b{I} - \frac{1}{n}\b{1}\b{1}^\top) = -2\,\b{HGH}
\end{align*}
\begin{align}\label{equation_linearKernel_and_distanceMAtrix_1}
\therefore~~~~~~~~ \b{HGH} = \b{H}\b{X}^\top\b{X}\b{H} = -\frac{1}{2} \b{HDH}.
\end{align}
Note that $(\b{I} - \frac{1}{n}\b{1}\b{1}^\top)\b{1} = \b{0}$ and $\b{1}^\top(\b{I} - \frac{1}{n}\b{1}\b{1}^\top) = \b{0}$ because removing the row mean of $\b{1}$ and column mean of of $\b{1}^\top$ results in the zero vectors, respectively.

If data $\b{X}$ are already centered, i.e., the mean has been removed ($\b{X} \gets \b{X}\b{H}$), Eq. (\ref{equation_linearKernel_and_distanceMAtrix_1}) becomes:
\begin{align}\label{equation_linearKernel_and_distanceMAtrix_2}
\b{X}^\top\b{X} = -\frac{1}{2} \b{HDH}.
\end{align}

\begin{corollary}
If using Eq. (\ref{equation_classical_MDS_X_decompose}) as Gram matrix, the classical MDS uses the Euclidean distance as its metric. Because of using Euclidean distance, the classical MDS using Gram matrix is a \underline{\textbf{linear}} subspace learning method. 
\end{corollary}
\begin{proof}
The Eq. (\ref{equation_classical_MDS_X_decompose}) in classical MDS is the eigenvalue decomposition of the Gram matrix $\b{X}^\top \b{X}$. According to Eq. (\ref{equation_linearKernel_and_distanceMAtrix_2}), this Gram matrix can be restated to an expression based on squared Euclidean distance. Hence, the classical MDS with Eq. (\ref{equation_classical_MDS_X_decompose}) uses Euclidean distance and is linear, consequently. 
\end{proof}

In Eq. (\ref{equation_linearKernel_and_distanceMAtrix_1}) or (\ref{equation_linearKernel_and_distanceMAtrix_2}), we can write a general kernel matrix \cite{hofmann2008kernel} rather than the double-centered Gram matrix, to have \cite{cox2008multidimensional}:
\begin{align}\label{equation_generalKernel_and_distanceMAtrix}
\mathbb{R}^{n \times n} \ni \b{K} = -\frac{1}{2} \b{HDH}.
\end{align}
Note that the classical MDS with Eq. (\ref{equation_classical_MDS_X_decompose}) is using a linear kernel $\b{X}^\top \b{X}$ for its kernel. This is another reason for why classical MDS with Eq. (\ref{equation_classical_MDS_X_decompose}) is a linear method. 
It is also noteworthy that Eq. (\ref{equation_generalKernel_and_distanceMAtrix}) can be used for unifying the spectral dimensionality reduction methods as special cases of kernel principal component analysis with different kernels. See \cite{ham2004kernel,bengio2004learning} and {\citep[Table 2.1]{strange2014open}} for more details. 

Comparing Eqs. (\ref{equation_linearKernel_and_distanceMAtrix_1}), (\ref{equation_linearKernel_and_distanceMAtrix_2}), and (\ref{equation_generalKernel_and_distanceMAtrix}) with Eq. (\ref{equation_classical_MDS_X_decompose}) shows that we can use a general kernel matrix, like Radial Basis Function (RBF) kernel, in classical MDS to have \textit{generalized classical MDS}. 
In summary, for embedding $\b{X}$ using classical MDS, the eigenvalue decomposition of the kernel matrix $\b{K}$ is obtained similar to Eq. (\ref{equation_classical_MDS_X_decompose}):
\begin{align}\label{equation_classical_MDS_kernel_decompose}
& \b{K} = \b{V} \b{\Delta} \b{V}^\top.
\end{align}
Then, using Eq. (\ref{equation_metric_MDS_Y}), $\b{Y} \in \mathbb{R}^{n \times n}$ is obtained. 
It is noteworthy that in this case, we are replacing $\b{X}^\top \b{X}$ with the kernel $\b{K} = \b{\Phi}(\b{X})^\top \b{\Phi}(\b{X})$ and then, according to Eqs. (\ref{equation_metric_MDS_Y}) and (\ref{equation_classical_MDS_kernel_decompose}), we have:
\begin{align}\label{equation_classical_MDS_kernel_inTermsOf_Y}
\b{K} = \b{Y}^\top \b{Y}.
\end{align}
Truncating the $\b{Y}$, obtained from Eq. (\ref{equation_metric_MDS_Y}), to have $\b{Y} \in \mathbb{R}^{p \times n}$, with the first (top) $p$ rows, gives us the $p$-dimensional embedding of the $n$ points. 
It is noteworthy that, because of using kernel in the generalized classical MDS, one can name it the \textit{kernel classical MDS}. 

\subsubsection{Equivalence of PCA and kernel PCA with Classical MDS and Generalized Classical MDS, Respectively}\label{section_PCA_MDS_equivalence}

\begin{proposition}
Classical MDS with Euclidean distance is equivalent to Principal Component Analysis (PCA). Moreover, the generalized classical MDS is equivalent to kernel PCA. 
\end{proposition}
\begin{proof}
On one hand, the Eq. (\ref{equation_classical_MDS_X_decompose}) can be obtained by the SVD of $\b{X}$. The projected data onto classical MDS subspace is obtained by Eq. (\ref{equation_metric_MDS_Y}) which is $\b{\Delta} \b{V}^\top$. 
On the other hand, according to {\citep[Eq. 42]{ghojogh2019unsupervised}}, the projected data onto PCA subspace is $\b{\Delta} \b{V}^\top$ where $\b{\Delta}$ and $\b{V}^\top$ are from the SVD of $\b{X}$. Comparing these shows that classical MDS is equivalent to PCA.  

Moreover, Eq. (\ref{equation_classical_MDS_kernel_decompose}) is the eigenvalue decomposition of the kernel matrix. The projected data onto the generalized classical MDS subspace is obtained by Eq. (\ref{equation_metric_MDS_Y}) which is $\b{\Delta} \b{V}^\top$. 
According to {\citep[Eq. 62]{ghojogh2019unsupervised}}, the projected data onto the kernel PCA subspace is $\b{\Delta} \b{V}^\top$ where $\b{\Delta}$ and $\b{V}^\top$ are from the eigenvalue decomposition of the kernel matrix; see {\citep[Eq. 61]{ghojogh2019unsupervised}}. Comparing these shows that the generalized classical MDS is equivalent to kernel PCA.  
\end{proof}

\subsection{Metric Multidimensional Scaling}

% The metric MDS is the default and usual MDS, which is often used \cite{jung2013lecture}. 

Recall that the classical MDS tries to preserve the similarities of points in the embedding space.
In later approaches after classical MDS, the cost function was changed to preserve the distances rather than the similarities \cite{lee2007nonlinear,bunte2012general}. 
\textit{Metric MDS} has this opposite view and tries to preserve the distances of points in the embedding space \cite{beals1968foundations}. For this, it minimizes the difference of distances of points in the input and embedding spaces \cite{ghodsi2006dimensionality}. The cost function in metric MDS is usually referred to as the \textit{stress function} \cite{mardia1978some,de2011multidimensional}. 
This method is named metric MDS because it uses distance metric in its optimization. 
The optimization in metric MDS is: 
\begin{equation}\label{equation_optimization_metric_MDS}
\begin{aligned}
& \underset{\{\b{y}_i\}_{i=1}^n}{\text{minimize}} \\ 
& c_2 := \Bigg( \frac{\sum_{i=1}^n \sum_{j=1, j < i}^n \big(d_x(\b{x}_i, \b{x}_j) - d_y(\b{y}_i, \b{y}_j)\big)^2}{\sum_{i=1}^n \sum_{j=1, j < i}^n d_x^2(\b{x}_i, \b{x}_j)} \Bigg)^{\frac{1}{2}},
\end{aligned}
\end{equation}
or, without the normalization factor:
\begin{equation}\label{equation_optimization_metric_MDS_2}
\begin{aligned}
& \underset{\{\b{y}_i\}_{i=1}^n}{\text{minimize}} \\ 
& c_2 := \Bigg( \sum_{i=1}^n \sum_{j=1, j < i}^n \big(d_x(\b{x}_i, \b{x}_j) - d_y(\b{y}_i, \b{y}_j)\big)^2 \Bigg)^{\frac{1}{2}},
\end{aligned}
\end{equation}
where $d_x(.,.)$ and $d_y(.,.)$ denote the distance metrics in the input and the embedded spaces, respectively. 

The Eqs. (\ref{equation_optimization_metric_MDS}) and (\ref{equation_optimization_metric_MDS_2}) use indices $j < i$ rather than $j \neq i$ because the distance metric is symmetric and it is not necessary to consider the distance of the $j$-th point from the $i$-th point when we already have considered the distance of the $i$-th point from the $j$-th point.
Note that in Eq. (\ref{equation_optimization_metric_MDS}) and (\ref{equation_optimization_metric_MDS_2}), $d_y$ is usually the Euclidean distance, i.e. $d_y = \|\b{y}_i - \b{y}_j\|_2$, while $d_x$ can be any valid metric distance such as the Euclidean distance. 

The optimization problem (\ref{equation_optimization_metric_MDS}) can be solved using either gradient descent or Newton's method. Note that the classical MDS is a linear method and has a closed-form solution; however, the metric and non-metric MDS methods are \underline{\textbf{nonlinear}} but do \textit{not have closed-form solutions} and should be solved iteratively. 
Note that in mathematics, whenever you get something, you lose something. Likewise, here, the method has become nonlinear but lost its closed form solution and became iterative. 

Inspired by \cite{sammon1969nonlinear}, we can use diagonal quasi-Newton's method for solving this optimization problem. If we consider the vectors component-wise, the diagonal quasi-Newton's method updates the solution as \cite{lee2007nonlinear}:
\begin{align}\label{equation_diagonal_Newton_method}
y_{i,k}^{(\nu+1)} := y_{i,k}^{(\nu)} - \eta\, \Big|\frac{\partial^2 c_2}{\partial y_{i,k}^2}\Big|^{-1}\, \frac{\partial c_2}{\partial y_{i,k}},
\end{align}
where $\eta$ is the learning rate, $y_{i,k}$ is the $k$-th element of the $i$-th embedded point $\mathbb{R}^p \ni \b{y}_i = [y_{i,1}, \dots, y_{i,p}]^\top$, and $|\!\cdot\!|$ is the absolute value guaranteeing that we move toward the minimum and not maximum in the Newton's method. 
If using gradient descent for solving the optimization, we update the solution as:
\begin{align}
y_{i,k}^{(\nu+1)} := y_{i,k}^{(\nu)} - \eta\, \frac{\partial c_2}{\partial y_{i,k}}.
\end{align}

\subsection{Non-Metric Multidimensional Scaling}

In \textit{non-metric MDS}, rather than using a distance metric, $d_y(\b{x}_i, \b{x}_j)$, for the distances between points in the embedding space, we use $f(d_y(\b{x}_i, \b{x}_j))$ where $f(.)$ is a non-parametric monotonic function. In other words, only the order of dissimilarities is important rather than the amount of dissimilarities \cite{agarwal2007generalized,jung2013lecture}:
\begin{equation}
\begin{aligned}
& d_y(\b{y}_i, \b{y}_j) \leq d_y(\b{y}_k, \b{y}_\ell) \Longleftrightarrow \\
&~~~~~~~~~~~~~~~ f(d_y(\b{y}_i, \b{y}_j)) \leq f(d_y(\b{x}_k, \b{y}_\ell)). 
\end{aligned}
\end{equation}

The optimization in non-metric MDS is \cite{agarwal2007generalized}: 
\begin{equation}\label{equation_optimization_nonmetric_MDS}
\begin{aligned}
& \underset{\{\b{y}_i\}_{i=1}^n}{\text{minimize}}~~~~ c_3 :=  \\ 
& \Bigg( \frac{\sum_{i=1}^n \sum_{j=1, j < i}^n \big(d_x(\b{x}_i, \b{x}_j) - f(d_y(\b{y}_i, \b{y}_j))\big)^2}{\sum_{i=1}^n \sum_{j=1, j < i}^n d_x^2(\b{x}_i, \b{x}_j)} \Bigg)^{\frac{1}{2}}.
\end{aligned}
\end{equation}
An examples of non-metric MDS is Smallest Space Analysis \cite{schlesinger1969smallest}. Another example is Kruskal's non-metric MDS or Shepard-Kruskal Scaling (SKS) \cite{kruskal1964non,kruskal1964multidimensional}. In Kruskal's non-metric MDS, the function $f(.)$ is the regression, where $f(d_y(\b{y}_i, \b{y}_j))$ is predicted from regression which preserves the order of dissimilarities \cite{holland2008non,agarwal2007generalized}. The Eq. (\ref{equation_optimization_nonmetric_MDS}) with $f(.)$ as the regression function, which is used in Kruskal's non-metric MDS, is called \textit{Stress-1 formula} \cite{agarwal2007generalized,holland2008non,jung2013lecture}.

\section{Sammon Mapping}\label{section_Sammon_mapping}

Sammon mapping \cite{sammon1969nonlinear} is a special case of metric MDS; hence, it is a \underline{\textbf{nonlinear}} method. It is probably correct to call this method the first proposed nonlinear method for manifold learning \cite{ghojogh2019roweis}. 

This method has different names in the literature such as \textit{Sammon's nonlinear mapping}, \textit{Sammon mapping}, and \textit{Nonlinear Mapping (NLM)} \cite{lee2007nonlinear}. Sammon originally named it NLM \cite{sammon1969nonlinear}. Its most well-known name is Sammon mapping. 

% Recall that in MDS, the cost function is:
% \begin{equation}
% \begin{aligned}
% & \underset{\b{Y}}{\text{minimize}}
% & & \sum_{i=1}^n \sum_{j=1}^n (\b{x}_i^\top \b{x}_j - \b{y}_i^\top \b{y}_j)^2,
% \end{aligned}
% \end{equation}
% which tries to preserve the similarities (gram matrices) of the data points in the embedded space.
% In later approaches after MDS, the cost function was changed to \cite{lee2007nonlinear,bunte2012general}:

The optimization problem in Sammon mapping is almost a weighted version of Eq. (\ref{equation_optimization_metric_MDS}), formulated as:
\begin{equation}\label{equation_MDS_distanceBased}
\begin{aligned}
& \underset{\{\b{y}_i\}_{i=1}^n}{\text{minimize}}
& & \frac{1}{a} \sum_{i=1}^n \sum_{j=1, j < i}^n w_{ij} \big(d_x(\b{x}_i, \b{x}_j) - d_y(\b{y}_i, \b{y}_j)\big)^2,
\end{aligned}
\end{equation}
where $w_{ij}$ is the weight and $a$ is the normalizing factor.
The $d_x(.,.)$ can be any metric but usually is considered to be Euclidean distance for simplicity \cite{lee2007nonlinear}. The $d_y(.,.)$, however, is Euclidean distance metric. 

In Sammon mapping, the weights and the normalizing factor in Eq. (\ref{equation_MDS_distanceBased}) are:
\begin{align}
& w_{ij} = \frac{1}{d_x(\b{x}_i, \b{x}_j)}, \label{equation_Sammon_weight} \\
& a = \sum_{i=1}^n \sum_{j=1, j < i}^n d_x(\b{x}_i, \b{x}_j). \label{equation_Sammon_normalizingFactor}
\end{align}
The weight $w_{ij}$ in Sammon mapping is giving more credit to the small distances (neighbor points) focusing on preserving the ``local'' structure of the manifold; hence it fits the manifold locally \cite{saul2003think}. 

Substituting Eqs. (\ref{equation_Sammon_weight}) and (\ref{equation_Sammon_normalizingFactor}) in Eq. (\ref{equation_MDS_distanceBased}) gives:
\begin{equation}
\begin{aligned}
& \underset{\b{Y}}{\text{minimize}}
& & c_4 := \frac{1}{\sum_{i=1}^n \sum_{j=1, j < i}^n d_x(\b{x}_i, \b{x}_j)} \times \\
& &&\sum_{i=1}^n \sum_{j=1, j < i}^n \frac{\big(d_x(\b{x}_i, \b{x}_j) - d_y(\b{y}_i, \b{y}_j)\big)^2}{d_x(\b{x}_i, \b{x}_j)}.
\end{aligned}
\end{equation}
% Note that in Sammon's paper \cite{sammon1969nonlinear}, the objective function is:
% \begin{equation}
% \begin{aligned}
% & \underset{\b{Y}}{\text{minimize}}
% & & c := \frac{1}{\sum_{i=1}^n \sum_{j=1, j < i}^n d_x(\b{x}_i, \b{x}_j)} \times \\
% & &&\sum_{i=1}^n \sum_{j=1, j < i}^n \frac{\big(d_x(\b{x}_i, \b{x}_j) - d_y(\b{y}_i, \b{y}_j)\big)^2}{d_x(\b{x}_i, \b{x}_j)},
% \end{aligned}
% \end{equation}
% which uses indices $j < i$ rather than $j \neq i$ because the distance metric is symmetric and it is not necessary to consider the distance of the $j$-th point from the $i$-th point when we already have considered the distance of the $i$-th point from the $j$-th point.

Sammon used diagonal quasi-Newton's method for solving this optimization problem \cite{sammon1969nonlinear}. Hence, Eq. (\ref{equation_diagonal_Newton_method}) is utilized. 
The learning rate $\eta$ is named the \textit{magic factor} in \cite{sammon1969nonlinear}.
For solving optimization, both gradient and second derivative are required. In the following, we derive these two. 

Note that, in practice, the classical MDS or PCA is used for initialization of points in Sammon mapping optimization. 

% If we consider the vectors component-wise, the diagonal quasi-Newton's method updates the solution as \cite{lee2007nonlinear}:
% \begin{align}
% y_{i,k}^{(\nu+1)} := y_{i,k}^{(\nu)} - \eta\, \Big|\frac{\partial^2 c}{\partial y_{i,k}^2}\Big|^{-1}\, \frac{\partial c}{\partial y_{i,k}},
% \end{align}
% where $y_{i,k}$ is the $k$-th element of the $i$-th embedded point $\mathbb{R}^p \ni \b{y}_i = [y_{i,1}, \dots, y_{i,p}]^\top$ and $|.|$ is the absolute value guaranteeing that we move toward the minimum and not maximum in the Newton's method. 

% and is suggested to be $\eta \approx 0.3 \text{ or } 0.4$ \cite{sammon1969nonlinear}.

\begin{proposition}
The gradient of the cost function $c$ with respect to $y_{i,k}$ is \cite{sammon1969nonlinear,lee2007nonlinear}:
\begin{align}\label{equation_Sammon_gradient}
&\frac{\partial c_4}{\partial y_{i,k}} \nonumber \\
&= \frac{-2}{a} \sum_{i=1}^n \sum_{j=1, j < i}^n \frac{d_x(\b{x}_i, \b{x}_j) - d_y(\b{y}_i, \b{y}_j)}{d_x(\b{x}_i, \b{x}_j)\, d_y(\b{x}_i, \b{x}_j)} (y_{i,k} - y_{j,k}).
\end{align}
\end{proposition}
\begin{proof}
Proof is according to \cite{lee2007nonlinear}.
According to chain rule, we have:
\begin{align*}
\frac{\partial c_4}{\partial y_{i,k}} = \frac{\partial c_4}{\partial d_y(\b{y}_i, \b{y}_j)} \times \frac{\partial d_y(\b{y}_i, \b{y}_j)}{\partial y_{i,k}}.
\end{align*}
The first derivative is:
\begin{align*}
\frac{\partial c_4}{\partial d_y(\b{y}_i, \b{y}_j)} = \frac{-2}{a} \sum_{i=1}^n \sum_{j=1, j < i}^n \frac{d_x(\b{x}_i, \b{x}_j) - d_y(\b{y}_i, \b{y}_j)}{d_x(\b{x}_i, \b{x}_j)},
\end{align*}
and using the chain rule, the second derivative is:
\begin{align*}
\frac{\partial d_y(\b{y}_i, \b{y}_j)}{\partial y_{i,k}} = \frac{\partial d_y(\b{y}_i, \b{y}_j)}{\partial d_y^2(\b{y}_i, \b{y}_j)} \times \frac{\partial d_y^2(\b{y}_i, \b{y}_j)}{\partial y_{i,k}}.
\end{align*}
We have:
\begin{align*}
\frac{\partial d_y(\b{y}_i, \b{y}_j)}{\partial d_y^2(\b{y}_i, \b{y}_j)} = 1 / \frac{\partial d_y^2(\b{y}_i, \b{y}_j)}{\partial d_y(\b{y}_i, \b{y}_j)} = 1 / (2 d_y(\b{y}_i, \b{y}_j)).
\end{align*}
Also we have:
\begin{align*}
d_y^2(\b{y}_i, \b{y}_j) = ||\b{y}_i - \b{y}_j||_2^2 = \sum_{k=1}^p (y_{i,k} - y_{j,k})^2.
\end{align*}
Therefore:
\begin{align*}
\frac{\partial d_y^2(\b{y}_i, \b{y}_j)}{\partial y_{i,k}} = 2\,(y_{i,k} - y_{j,k}),
\end{align*}
Therefore:
\begin{align}\label{equation_Sammon_derivative_dY}
\therefore ~~~ \frac{\partial d_y(\b{y}_i, \b{y}_j)}{\partial y_{i,k}} = \frac{y_{i,k} - y_{j,k}}{d_y(\b{y}_i, \b{y}_j)}.
\end{align}
Finally, we have:
\begin{align*}
&\therefore ~~~ \frac{\partial c_4}{\partial y_{i,k}} \nonumber \\
&= \frac{-2}{a} \sum_{i=1}^n \sum_{j=1, j < i}^n \frac{d_x(\b{x}_i, \b{x}_j) - d_y(\b{y}_i, \b{y}_j)}{d_x(\b{x}_i, \b{x}_j)\, d_y(\b{x}_i, \b{x}_j)} (y_{i,k} - y_{j,k}),
\end{align*}
which is the gradient mentioned in the proposition. Q.E.D.
\end{proof}

\begin{proposition}
The second derivative of the cost function $c$ with respect to $y_{i,k}$ is \cite{sammon1969nonlinear,lee2007nonlinear}:
\begin{alignat}{2}
&\frac{\partial^2 c_4}{\partial y_{i,k}^2} = \,&&\frac{-2}{a} \sum_{i=1}^n \sum_{j=1, j < i}^n \Big( \frac{d_x(\b{x}_i, \b{x}_j) - d_y(\b{y}_i, \b{y}_j)}{d_x(\b{x}_i, \b{x}_j)\, d_y(\b{x}_i, \b{x}_j)} \nonumber \\
& && - \frac{(y_{i,k} - y_{j,k})^2}{d_y^3(\b{y}_i, \b{y}_j)} \Big).
\end{alignat}
\end{proposition}
\begin{proof}
We have:
\begin{align*}
\frac{\partial^2 c_4}{\partial y_{i,k}^2} = \frac{\partial}{\partial y_{i,k}} \Big(\frac{\partial c_4}{\partial y_{i,k}}\Big),
\end{align*}
where $\partial c_4 / \partial y_{i,k}$ is Eq. (\ref{equation_Sammon_gradient}). Therefore:
\begin{align*}
\frac{\partial^2 c_4}{\partial y_{i,k}^2} = &\frac{-2}{a} \sum_{i=1}^n \sum_{j=1, j < i}^n \frac{\partial}{\partial y_{i,k}} \\
&\Big( \frac{d_x(\b{x}_i, \b{x}_j) - d_y(\b{y}_i, \b{y}_j)}{d_x(\b{x}_i, \b{x}_j)\, d_y(\b{x}_i, \b{x}_j)} (y_{i,k} - y_{j,k}) \Big).
\end{align*}
We have:
\begin{align*}
&\frac{\partial}{\partial y_{i,k}} \Big( \frac{d_x(\b{x}_i, \b{x}_j) - d_y(\b{y}_i, \b{y}_j)}{d_x(\b{x}_i, \b{x}_j)\, d_y(\b{x}_i, \b{x}_j)} (y_{i,k} - y_{j,k}) \Big) \\
&= (y_{i,k} - y_{j,k}) \frac{\partial}{\partial y_{i,k}} \Big( \frac{d_x(\b{x}_i, \b{x}_j) - d_y(\b{y}_i, \b{y}_j)}{d_x(\b{x}_i, \b{x}_j)\, d_y(\b{x}_i, \b{x}_j)} \Big) \\
& + \frac{d_x(\b{x}_i, \b{x}_j) - d_y(\b{y}_i, \b{y}_j)}{d_x(\b{x}_i, \b{x}_j)\, d_y(\b{x}_i, \b{x}_j)} \underbrace{\frac{\partial}{\partial y_{i,k}} (y_{i,k} - y_{j,k})}_{=1}.
\end{align*}
Note that:
\begin{align*}
&\frac{\partial}{\partial y_{i,k}} \Big( \frac{d_x(\b{x}_i, \b{x}_j) - d_y(\b{y}_i, \b{y}_j)}{d_x(\b{x}_i, \b{x}_j)\, d_y(\b{x}_i, \b{x}_j)} \Big) \\
&= \frac{1}{d_x(\b{x}_i, \b{x}_j)} \frac{\partial}{\partial y_{i,k}} \Big( \frac{d_x(\b{x}_i, \b{x}_j) - d_y(\b{y}_i, \b{y}_j)}{d_y(\b{x}_i, \b{x}_j)} \Big) \\
&= \frac{1}{d_x(\b{x}_i, \b{x}_j)} \frac{\partial}{\partial y_{i,k}} \Big( \frac{d_x(\b{x}_i, \b{x}_j)}{d_y(\b{x}_i, \b{x}_j)} - 1 \Big) \\
&= \underbrace{\frac{d_x(\b{x}_i, \b{x}_j)}{d_x(\b{x}_i, \b{x}_j)}}_{=1} \frac{\partial}{\partial y_{i,k}} \Big( \frac{1}{d_y(\b{x}_i, \b{x}_j)}\Big) - \underbrace{\frac{\partial}{\partial y_{i,k}} (1)}_{=0} \\
&= \frac{-1}{d_y^2(\b{x}_i, \b{x}_j)} \frac{\partial}{\partial y_{i,k}} ( d_y(\b{x}_i, \b{x}_j)) \\
&\overset{(\ref{equation_Sammon_derivative_dY})}{=} \frac{-1}{d_y^2(\b{x}_i, \b{x}_j)} \frac{y_{i,k} - y_{j,k}}{d_y(\b{y}_i, \b{y}_j)}.
\end{align*}
Therefore:
\begin{align*}
&\therefore ~~~ \frac{\partial}{\partial y_{i,k}} \Big( \frac{d_x(\b{x}_i, \b{x}_j) - d_y(\b{y}_i, \b{y}_j)}{d_x(\b{x}_i, \b{x}_j)\, d_y(\b{x}_i, \b{x}_j)} (y_{i,k} - y_{j,k}) \Big) \\
&= \frac{-(y_{i,k} - y_{j,k})^2}{d_y^3(\b{y}_i, \b{y}_j)} + \frac{d_x(\b{x}_i, \b{x}_j) - d_y(\b{y}_i, \b{y}_j)}{d_x(\b{x}_i, \b{x}_j)\, d_y(\b{x}_i, \b{x}_j)}.
\end{align*}
Therefore:
\begin{alignat*}{2}
&\therefore ~~~ \frac{\partial^2 c_4}{\partial y_{i,k}^2} = \,&&\frac{-2}{a} \sum_{i=1}^n \sum_{j=1, j < i}^n \Big( \frac{d_x(\b{x}_i, \b{x}_j) - d_y(\b{y}_i, \b{y}_j)}{d_x(\b{x}_i, \b{x}_j)\, d_y(\b{x}_i, \b{x}_j)} \nonumber \\
& && - \frac{(y_{i,k} - y_{j,k})^2}{d_y^3(\b{y}_i, \b{y}_j)} \Big),
\end{alignat*}
which is the derivative mentioned in the proposition. Q.E.D.
\end{proof}

It is noteworthy that for better time complexity of the Sammon mapping, one can use the $k$-Nearest Neighbors ($k$NN) rather than the whole data \cite{ghojogh2020quantile}:
\begin{equation}\label{equation_Sammon_with_neighbors}
\begin{aligned}
& \underset{\{\b{y}_i\}_{i=1}^n}{\text{minimize}}
& & \frac{1}{a} \sum_{i=1}^n \sum_{j \in \mathcal{N}_i}^n w_{ij} \big(d_x(\b{x}_i, \b{x}_j) - d_y(\b{y}_i, \b{y}_j)\big)^2,
\end{aligned}
\end{equation}
where $\mathcal{N}_i$ denotes the set of indices of $k$NN of the $i$-th point. 

\section{Isomap}\label{section_Isomap}

\subsection{Isomap}

\textit{Isomap} \cite{tenenbaum2000global} is a special case of the generalized classical MDS, explained in Section \ref{section_generalized_classical_MDS}. Rather than the Euclidean distance, Isomap uses an approximation of the geodesic distance. 
As was explained, the classical MDS is linear; hence, it cannot capture the nonlinearity of the manifold. 
Isomap makes use of the geodesic distance to make the generalized classical MDS nonlinear. 

\subsubsection{Geodesic Distance}

The \textit{geodesic distance} is the length of shortest path between two points on the possibly curvy manifold. 
It is ideal to use the geodesic distance; however, calculation of the geodesic distance is very difficult because it requires traversing from a point to another point on the manifold. This calculation requires differential geometry and Riemannian manifold calculations \cite{aubin2001course}.   
Therefore, Isomap approximates the geodesic distance by piece-wise Euclidean distances. It finds the $k$-Nearest Neighbors ($k$NN) graph of dataset. Then, the shortest path between two points, through their neighbors, is found using a shortest-path algorithm such as the Dijkstra algorithm or the Floyd-Warshal algorithm \cite{cormen2009introduction}. A sklearn function in python for this is ``graph\_shortest\_path'' from the package ``sklearn.utils.graph\_shortest\_path''. 
Note that the approximated geodesic distance is also refered to as the \textit{curvilinear distance} \cite{lee2002curvilinear}. 
The approximated geodesic distance can be formulated as \cite{bengio2004out}:
\begin{equation}\label{equation_geodesic_distance_matrix}
\b{D}^{(g)}_{ij} := \min_{\b{r}} \sum_{i = 2}^{l} \|\b{r}_i - \b{r}_{i+1}\|_2,
\end{equation}
where $l \geq 2$ is the length of sequence of points $\b{r}_i \in \{\b{x}_i\}_{i=1}^n$ and $\b{D}^{(g)}_{ij}$ denotes the $(i,j)$-th element of the geodesic distance matrix $\b{D}^{(g)} \in \mathbb{R}^{n \times n}$. 

\begin{figure}[!t]
\centering
\includegraphics[width=2.2in]{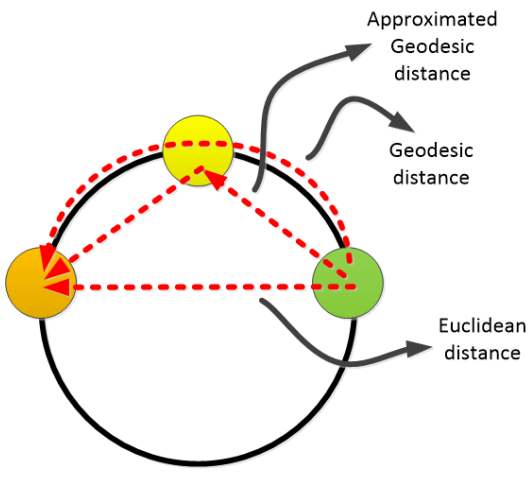}
\caption{An example of the Euclidean distance, geodesic distance, and approximated geodesic distance using piece-wise Euclidean distances.}
\label{figure_geodesic}
\end{figure}

An example of the Euclidean distance, geodesic distance, and the approximated geodesic distance using piece-wise Euclidean distances can be seen in Fig. \ref{figure_geodesic}. A real-world example is the distance between Toronto and Athens. The Euclidean distance is to dig the Earth from Toronto to reach Athens directly. The geodesic distance is to move from Toronto to Athens on the curvy Earth by the shortest path between two cities. The approximated geodesic distance is to dig the Earth from Toronto to London in UK, then dig from London to Frankfurt in Germany, then dig from Frankfurt to Rome in Italy, then dig from Rome to Athens. Calculations of lengths of paths in the approximated geodesic distance is much easier than the geodesic distance. 

\subsubsection{Isomap Formulation}

As was mentioned before, Isomap is a special case of the generalized classical MDS with the geodesic distance used. Hence, Isomap uses Eq. (\ref{equation_generalKernel_and_distanceMAtrix}) as:
\begin{align}\label{equation_Kernel_Isomap}
\mathbb{R}^{n \times n} \ni \b{K} = -\frac{1}{2} \b{H}\b{D}^{(g)}\b{H}.
\end{align}
It then uses Eqs. (\ref{equation_classical_MDS_kernel_decompose}) and (\ref{equation_metric_MDS_Y}) to embed the data. 
As Isomap uses the nonlinear geodesic distance in its kernel calculation, it is a \underline{\textbf{nonlinear}} method. 

\subsection{Kernel Isomap}\label{section_kernel_Isomap}

% \subsubsection{Kernel Isomap -- Version 1}

Consider $\b{K}(\b{D})$ to be Eq. (\ref{equation_generalKernel_and_distanceMAtrix}). Consequently, we have:
\begin{align}\label{equation_generalKernel_and_distanceMAtrix_squared}
\mathbb{R}^{n \times n} \ni \b{K}(\b{D}^2) = -\frac{1}{2} \b{H}\b{D}^2\b{H},
\end{align}
where $\b{D}$ is the geodesic distance matrix, defined by Eq. (\ref{equation_geodesic_distance_matrix}). 

Define the following equation {\citep[Section 2.2.8]{cox2008multidimensional}}:
\begin{align}\label{equation_kernelIsomap1_K_prime}
\mathbb{R}^{n \times n} \ni \b{K}' := \b{K}(\b{D}^2) + 2c\b{K}(\b{D}) + \frac{1}{2} c^2 \b{H}. 
\end{align}
According to \cite{cailliez1983analytical}, $\b{K}'$ is guaranteed to be positive semi-definite for $c \geq c^*$ where $c^*$ is the largest eigenvalue of the following matrix:
\begin{align}
\begin{bmatrix}
\b{0} & 2\b{K}(\b{D}^2) \\
-\b{I} & -4\b{K}(\b{D}) 
\end{bmatrix}
\in \mathbb{R}^{2n \times 2n}.
\end{align}
\textit{Kernel Isomap} \cite{choi2004kernel} chooses a value $c \geq c^*$ and uses $\b{K}'$ in Eq. (\ref{equation_classical_MDS_kernel_decompose}) and then uses Eq. (\ref{equation_metric_MDS_Y}) for embedding the data. 

% \subsubsection{Kernel Isomap -- Version 2}

% kernel isomap and out-of-sample \cite{gisbrecht2015parametric}

\section{Out-of-sample Extensions for MDS and Isomap}\label{section_outOfSample}

So far, we embedded the training dataset $\{\b{x}_i \in \mathbb{R}^d\}_{i=1}^n$ or $\b{X} = [\b{x}_1, \dots, \b{x}_n] \in \mathbb{R}^{d \times n}$ to have their embedding $\{\b{y}_i \in \mathbb{R}^p\}_{i=1}^n$ or $\b{Y} = [\b{y}_1, \dots, \b{y}_n] \in \mathbb{R}^{p \times n}$.
Assume we have some out-of-sample (test data), denoted by $\{\b{x}_i^{(t)} \in \mathbb{R}^d\}_{i=1}^{n_t}$ or $\b{X}_t = [\b{x}_1^{(t)}, \dots, \b{x}_n^{(t)}] \in \mathbb{R}^{d \times n_t}$. We want to find their embedding $\{\b{y}_i^{(t)} \in \mathbb{R}^p\}_{i=1}^{n_t}$ or $\b{Y}_t = [\b{y}_1^{(t)}, \dots, \b{y}_n^{(t)}] \in \mathbb{R}^{p \times n_t}$ after the training phase. 

\subsection{Out of Sample for Isomap and MDS Using Eigenfunctions}

\subsubsection{Eigenfunctions}

Consider a Hilbert space $\mathcal{H}_p$ of functions with the inner product $\langle f,g \rangle = \int f(x) g(x) p(x) dx$ with density function $p(x)$. In this space, we can consider the kernel function $K_p$:
\begin{align}
(K_p f)(x) = \int K(x,y)\, f(y)\, p(y)\, dy,
\end{align}
where the density function can be approximated empirically. 
The \textit{eigenfunction decomposition} is defined to be \cite{bengio2004learning,bengio2004out}:
\begin{align}
(K_p f_k)(x) = \delta'_k f_k(x),
\end{align}
where $f_k(x)$ is the $k$-th \textit{eigenfunction} and $\delta'_k$ is the corresponding eigenvalue.
If we have the eigenvalue decomposition \cite{ghojogh2019eigenvalue} for the kernel matrix $\b{K}$, we have $\b{K} \b{v}_k = \delta_k \b{v}_k$ (see Eq. (\ref{equation_classical_MDS_kernel_decompose})) where $\b{v}_k$ is the $k$-th eigenvector and $\delta_k$ is the corresponding eigenvalue. According to {\citep[Proposition 1]{bengio2004out}}, we have $\delta'_k = (1/n) \delta_k$. 

\subsubsection{Embedding Using Eigenfunctions}

\begin{proposition}
If $v_{ki}$ is the $i$-th element of the $n$-dimensional vector $\b{v}_k$ and $k(\b{x}, \b{x}_i)$ is the kernel between vectors $\b{x}$ and $\b{x}_i$, the eigenfunction for the point $\b{x}$ and the $i$-th training point $\b{x}_i$ are:
\begin{align}
f_k(\b{x}) &= \frac{\sqrt{n}}{\delta_k} \sum_{i=1}^n v_{ki}\, \breve{k}_t(\b{x}_i, \b{x}), \\
f_k(\b{x}_i) &= \sqrt{n}\, v_{ki}, 
\end{align}
respectively, where $\breve{k}_t(\b{x}_i, \b{x})$ is the centered kernel between training set and the out-of-sample point $\b{x}$. 

Let the MDS or Isomap embedding of the point $\b{x}$ be $\mathbb{R}^p \ni \b{y}(\b{x}) = [y_1(\b{x}), \dots, y_p(\b{x})]^\top$. The $k$-th dimension of this embedding is:
\begin{align}\label{equation_embedding_eigenfunction}
y_k(\b{x}) &= \sqrt{\delta_k}\, \frac{f_k(\b{x})}{\sqrt{n}} = \frac{1}{\sqrt{\delta_k}} \sum_{i=1}^n v_{ki}\, \breve{k}_t(\b{x}_i, \b{x}).
\end{align}
\end{proposition}

\begin{proof}
This proposition is taken from {\citep[Proposition 1]{bengio2004out}}. For proof, refer to {\citep[Proposition 1]{bengio2004learning}}, {\citep[Proposition 1]{bengio2006spectral}}, and {\citep[Proposition 1 and Theorem 1]{bengio2003spectral}}. 
More complete proofs can be found in \cite{bengio2003learning}. 
\end{proof}

If we have a set of $n_t$ out-of-sample data points, $\breve{k}_t(\b{x}_i, \b{x})$ is an element of the centered out-of-sample kernel (see {\citep[Appendix C]{ghojogh2019unsupervised}}):
\begin{align}
\mathbb{R}^{n \times n_t} \ni \breve{\b{K}}_t &= \b{K}_t - \frac{1}{n} \b{1}_{n \times n} \b{K}_t - \frac{1}{n} \b{K} \b{1}_{n \times n_t} \nonumber \\
&~~~~ + \frac{1}{n^2} \b{1}_{n \times n} \b{K} \b{1}_{n \times n_t}, \label{equation_centered_outOfSample_kernel}
\end{align}
where $\b{1} := [1, 1, \dots, 1]^\top$, $\b{K}_t \in \mathbb{R}^{n \times n_t}$ is the not necessarily centered out-of-sample kernel, and $\b{K} \in \mathbb{R}^{n \times n}$ is the training kernel.

\subsubsection{Out-of-sample Embedding}

One can use Eq. (\ref{equation_embedding_eigenfunction}) to embed the $i$-th out-of-sample data point $\b{x}_i^{(t)}$. For this purpose, $\b{x}_i^{(t)}$ should be used in place of $\b{x}$ in Eq. (\ref{equation_embedding_eigenfunction}). 

Note that Eq. (\ref{equation_embedding_eigenfunction}) requires Eq. (\ref{equation_centered_outOfSample_kernel}).
In MDS and Isomap, $\b{K}$ is obtained by the linear kernel, $\b{X}^\top \b{X}$, and Eq. (\ref{equation_Kernel_Isomap}), respectively. 
Also, the out-of-sample kernel $\b{K}_t$ in MDS is obtained by the linear kernel between the training and out-of-sample data, i.e., $\b{X}^\top \b{X}_t$. In Isomap, the kernel $\b{K}_t$ is obtained by centering the geodesic distance matrix (see Eq. (\ref{equation_Kernel_Isomap})) where the geodesic distance matrix between the training and out-of-sample data is used. In calculation of this geodesic distance matrix, merely the training data points, and not the test points, should be used as the intermediate points in paths \cite{bengio2004out}. 

It is shown in {\citep[Corollary 1]{bengio2004out}} that using the geodesic distance with only training data as intermediate points, for teh sake of out-of-sample embedding in Isomap, is equivalent to the \textit{landmark Isomap} method \cite{de2003global}:
\begin{align}\label{equation_Isomap_outOfSample}
y_k(\b{x}) = \frac{1}{2 \sqrt{\delta_k}} \sum_{i=1}^n v_{ki} (\b{D}^{(g)}_\text{avg} - \b{D}^{(g)}_t(\b{x}_i, \b{x})),
\end{align}
where $\b{D}^{(g)}_\text{avg}$ denotes the average geodesic distance between the training points and $\b{D}^{(g)}_t$ is the geodesic distance between the $i$-th training point $\b{x}_i$ and the out-of-sample point $\b{x}$, in which the training set is used for intermediate points. 
Hence, one can use Eq. (\ref{equation_Isomap_outOfSample}) for out-of-sample embedding in Isomap.

It is noteworthy that in addition to the out-of-sample extension using eigenfunctions \cite{bengio2004out}, there exist some other methods for out-of-sample extension of MDS and Isomap \cite{bunte2012general,strange2011generalised}, which we pass by in this paper. 

% \subsection{Out of Sample for Kernel Isomap Using Kernel PCA}

% Kernel Isomap \cite{choi2004kernel} uses kernel PCA \cite{ghojogh2019unsupervised} for out-of-sample data. 
% First, the kernel between the training data $\b{X} \in \mathbb{R}^{d \times n}$ and the out-of-sample data $\b{X}_t \in \mathbb{R}^{d \times n_t}$ is calculated to have $\b{K}_t \in \mathbb{R}^{n \times n_t}$. This kernel can be centered in the feature space as {\citep[Appendix C]{ghojogh2019unsupervised}}:
% \begin{align}
% \mathbb{R}^{n \times n_t} \ni \breve{\b{K}}_t &= \b{K}_t - \frac{1}{n} \b{1}_{n \times n} \b{K}_t - \frac{1}{n} \b{K}' \b{1}_{n \times n_t} \nonumber \\
% &~~~~ + \frac{1}{n^2} \b{1}_{n \times n} \b{K}' \b{1}_{n \times n_t}, 
% \end{align}
% where $\b{1} := [1, 1, \dots, 1]^\top$.

% According to {\citep[Section 4.2]{ghojogh2019unsupervised}}, the eigenvalue problem for the kernel $\b{K}'$ (Eq. (\ref{equation_kernelIsomap1_K_prime})) is obtained:
% \begin{align}\label{equation_kernelPCA_eigen_centered_kernel_2}
% \b{K}' \b{V} = \b{V} \b{\Sigma}^2,
% \end{align}
% which is equivalent to Eq. (\ref{equation_classical_MDS_kernel_decompose}) where $\b{K}'$ is used rather than $\b{K}$ and $\b{\Delta} = \b{\Sigma}^2$. 
% Finally, according to {\citep[Section 4.4]{ghojogh2019unsupervised}}, using the leading $p$ eigenvectors, the embedding $\b{Y}_t$ of the out-of-sample data $\b{X}_t$ is calculated as:
% \begin{align}
% \mathbb{R}^{p \times n_t} \ni \b{Y}_t = \b{\Sigma}^{-1}\b{V}^\top\breve{\b{K}}_t,
% \end{align}

\subsection{Out of Sample for Isomap, Kernel Isomap, and MDS Using Kernel Mapping}

There is a kernel mapping method \cite{gisbrecht2012out,gisbrecht2015parametric} to embed the out-of-sample data in Isomap, kernel Isomap, and MDS. 
We introduce this method here.

We define a map which maps any data point as $\b{x} \mapsto \b{y}(\b{x})$, where:
\begin{align}\label{equation_kernel_tSNE_map}
\mathbb{R}^p \ni \b{y}(\b{x}) := \sum_{j=1}^n \b{\alpha}_j\, \frac{k(\b{x}, \b{x}_j)}{\sum_{\ell=1}^n k(\b{x}, \b{x}_{\ell})},
\end{align}
and $\b{\alpha}_j \in \mathbb{R}^p$, and $\b{x}_j$ and $\b{x}_{\ell}$ denote the $j$-th and $\ell$-th training data point.
The $k(\b{x}, \b{x}_j)$ is a kernel such as the Gaussian kernel:
\begin{align}
k(\b{x}, \b{x}_j) = \exp(\frac{-||\b{x} - \b{x}_j||_2^2}{2\, \sigma_j^2}),
\end{align}
where $\sigma_j$ is calculated as \cite{gisbrecht2015parametric}:
\begin{align}
\sigma_j := \gamma \times \min_{i}(||\b{x}_j - \b{x}_i||_2),
\end{align}
where $\gamma$ is a small positive number.

Assume we have already embedded the training data points using MDS (see Section \ref{section_MDS}), Isomap (see Section \ref{section_Isomap}), or kernel Isomap (see Section \ref{section_kernel_Isomap}); therefore, the set $\{\b{y}_i\}_{i=1}^n$ is available.
If we map the training data points, we want to minimize the following least-squares cost function in order to get $\b{y}(\b{x}_i)$ close to $\b{y}_i$ for the $i$-th training point:
\begin{equation}
\begin{aligned}
& \underset{\b{\alpha}_j\text{'s}}{\text{minimize}}
& & \sum_{i=1}^n ||\b{y}_i - \b{y}(\b{x}_i)||_2^2,
\end{aligned}
\end{equation}
where the summation is over the training data points.
We can write this cost function in matrix form as below:
\begin{equation}\label{equation_kernel_tSNE_leastSquares}
\begin{aligned}
& \underset{\b{A}}{\text{minimize}}
& & ||\b{Y} - \b{K}''\b{A}||_F^2,
\end{aligned}
\end{equation}
where $\mathbb{R}^{n \times p} \ni \b{Y} := [\b{y}_1, \dots, \b{y}_n]^\top$ and $\mathbb{R}^{n \times p} \ni \b{A} := [\b{\alpha}_1, \dots, \b{\alpha}_n]^\top$. 
The $\b{K}'' \in \mathbb{R}^{n \times n}$ is the kernel matrix whose $(i,j)$-th element is defined to be:
\begin{align}
\b{K}''(i,j) := \frac{k(\b{x}_i, \b{x}_j)}{\sum_{\ell=1}^n k(\b{x}_i, \b{x}_{\ell})}.
\end{align}
The Eq. (\ref{equation_kernel_tSNE_leastSquares}) is always non-negative; thus, its smallest value is zero.
Therefore, the solution to this equation is:
\begin{align}
\b{Y} - \b{K}''\b{A} = \b{0} &\implies \b{Y} = \b{K}''\b{A} \nonumber \\
&\overset{(a)}{\implies} \b{A} = \b{K}''^{\dagger}\, \b{Y}, \label{equation_kernel_tSNE_A_matrix}
\end{align}
where $\b{K}''^{\dagger}$ is the pseudo-inverse of $\b{K}''$:
\begin{align}
\b{K}''^{\dagger} = (\b{K}''^\top \b{K}'')^{-1} \b{K}''^\top,
\end{align}
and $(a)$ is because $\b{K}''^{\dagger}\,\b{K}'' = \b{I}$.

Finally, the mapping of Eq. (\ref{equation_kernel_tSNE_map}) for the $n_t$ out-of-sample data points is:
\begin{align}\label{equation_kernel_tSNE_outOfSample_Y}
\b{Y}_t = \b{K}''_t\,\b{A}, 
\end{align}
where the $(i,j)$-th element of the out-of-sample kernel matrix $\b{K}''_t \in \mathbb{R}^{n_t \times n}$ is:
\begin{align}
\b{K}''_t(i,j) := \frac{k(\b{x}_i^{(t)}, \b{x}_j)}{\sum_{\ell=1}^n k(\b{x}_i^{(t)}, \b{x}_{\ell})},
\end{align}
where $\b{x}_i^{(t)}$ is the $i$-th out-of-sample data point, and $\b{x}_j$ and $\b{x}_{\ell}$ are the $j$-th and $\ell$-th training data points.

\section{Landmark MDS and Landmark Isomap for Big Data Embedding}\label{section_landmark_methods}

Nystrom approximation, introduced below, can be used to make the spectral methods such as MDS and Isomap scalable and suitable for big data embedding. 

\subsection{Nystrom Approximation}

\textit{Nystrom approximation} is a technique used to approximate a positive semi-definite matrix using merely a subset of its columns (or rows) \cite{williams2001using}. 
Consider a positive semi-definite matrix $\mathbb{R}^{n \times n} \ni \b{K} \succeq 0$ whose parts are:
\begin{align}\label{equation_Nystrom_partions}
\mathbb{R}^{n \times n} \ni \b{K} = 
\left[
\begin{array}{c|c}
\b{A} & \b{B} \\
\hline
\b{B}^\top & \b{C}
\end{array}
\right],
\end{align}
where $\b{A} \in \mathbb{R}^{m \times m}$, $\b{B} \in \mathbb{R}^{m \times (n-m)}$, and $\b{C} \in \mathbb{R}^{(n-m) \times (n-m)}$ in which $m \ll n$. 

The Nystrom approximation says if we have the small parts of this matrix, i.e. $\b{A}$ and $\b{B}$, we can approximate $\b{C}$ and thus the whole matrix $\b{K}$. The intuition is as follows. Assume $m=2$ (containing two points, a and b) and $n=5$ (containing three other points, c, d, and e). If we know the similarity (or distance) of points a and b from one another, resulting in matrix $\b{A}$, as well as the similarity (or distance) of points c, d, and e from a and b, resulting in matrix $\b{B}$, we cannot have much freedom on the location of c, d, and e, which is the matrix $\b{C}$. This is because of the positive semi-definiteness of the matrix $\b{K}$. 
The points selected in submatrix $\b{A}$ are named \textit{landmarks}. Note that the landmarks can be selected randomly from the columns/rows of matrix $\b{K}$ and, without loss of generality, they can be put together to form a submatrix at the top-left corner of matrix. 

As the matrix $\b{K}$ is positive semi-definite, by definition, it can be written as $\b{K} = \b{O}^\top \b{O}$. If we take $\b{O} = [\b{R}, \b{S}]$ where $\b{R}$ are the selected columns (landmarks) of $\b{O}$ and $\b{S}$ are the other columns of $\b{O}$. We have:
\begin{align}
\b{K} &= \b{O}^\top \b{O} = 
\begin{bmatrix}
\b{R}^\top \\
\b{S}^\top
\end{bmatrix}
[\b{R}, \b{S}] \label{equation_Nystrom_kernel_OtransposeO} \\
&= 
\begin{bmatrix}
\b{R}^\top \b{R} & \b{R}^\top \b{S} \\
\b{S}^\top \b{R} & \b{S}^\top \b{S}
\end{bmatrix}
\overset{(\ref{equation_Nystrom_partions})}{=} 
\begin{bmatrix}
\b{A} & \b{B} \\
\b{B}^\top & \b{C}
\end{bmatrix}.
\end{align}
Hence, we have $\b{A} = \b{R}^\top \b{R}$. The eigenvalue decomposition \cite{ghojogh2019eigenvalue} of $\b{A}$ gives:
\begin{align}
&\b{A} = \b{U} \b{\Sigma} \b{U}^\top \label{equation_Nystrom_A_eig_decomposition} \\
&\implies \b{R}^\top \b{R} = \b{U} \b{\Sigma} \b{U}^\top \implies \b{R} = \b{\Sigma}^{(1/2)} \b{U}^\top. \label{equation_Nystrom_R}
\end{align}
Moreover, we have $\b{B} = \b{R}^\top \b{S}$ so we have:
\begin{align}
&\b{B} = (\b{\Sigma}^{(1/2)} \b{U}^\top)^\top \b{S} = \b{U} \b{\Sigma}^{(1/2)} \b{S} \nonumber \\
&\overset{(a)}{\implies} \b{U}^\top \b{B} = \b{\Sigma}^{(1/2)} \b{S} \implies \b{S} = \b{\Sigma}^{(-1/2)} \b{U}^\top \b{B}, \label{equation_Nystrom_S}
\end{align}
where $(a)$ is because $\b{U}$ is orthogonal (in the eigenvalue decomposition). 
Finally, we have:
\begin{align}
\b{C} &= \b{S}^\top \b{S} = \b{B}^\top \b{U} \b{\Sigma}^{(-1/2)} \b{\Sigma}^{(-1/2)} \b{U}^\top \b{B} \nonumber \\
&= \b{B}^\top \b{U} \b{\Sigma}^{-1} \b{U}^\top \b{B} \overset{(\ref{equation_Nystrom_A_eig_decomposition})}{=} \b{B}^\top \b{A}^{-1} \b{B}. \label{equation_Nystrom_C}
\end{align}
Therefore, Eq. (\ref{equation_Nystrom_partions}) becomes:
\begin{align}\label{equation_Nystrom_partions_withDetails}
\b{K} \approx 
\left[
\begin{array}{c|c}
\b{A} & \b{B} \\
\hline
\b{B}^\top & \b{B}^\top \b{A}^{-1} \b{B}
\end{array}
\right].
\end{align}

\begin{proposition}
By increasing $m$, the approximation of Eq. (\ref{equation_Nystrom_partions_withDetails}) becomes more accurate. 
If rank of $\b{K}$ is at most $m$, this approximation is exact. 
\end{proposition}
\begin{proof}
In Eq. (\ref{equation_Nystrom_C}), we have the inverse of $\b{A}$. In order to have this inverse, the matrix $\b{A}$ must not be singular. For having a full-rank $\b{A} \in \mathbb{R}^{m \times m}$, the rank of $\b{A}$ should be $m$. This results in $m$ to be an upper bound on the rank of $\b{K}$ and a lower bound on the number of landmarks. In practice, it is recommended to use more number of landmarks for more accurate approximation but there is a trade-off with the speed. 
\end{proof}

\begin{corollary}
As we usually have $m \ll n$, the Nystrom approximation works well especially for the low-rank matrices \cite{kishore2017literature}. Usually, because of the manifold hypothesis, data fall on a submanifold; hence, usually, the kernel (similarity) matrix  or the distance matrix has a low rank. Therefore, the Nystrom approximation works well for many kernel-based or distance-based manifold learning methods. 
\end{corollary}

\subsection{Using Kernel Approximation in Landmark MDS}

Consider Eq. (\ref{equation_Nystrom_partions}) or (\ref{equation_Nystrom_partions_withDetails}) as the partitions of the kernel matrix $\b{K}$. Note that the (Mercer) kernel matrix is positive semi-definite so the Nystrom approximation can be applied for kernels.  

Recall that Eq. (\ref{equation_classical_MDS_kernel_decompose}) decomposes the kernel matrix into eigenvectors and then Eq. (\ref{equation_metric_MDS_Y}) embeds data. However, for big data, the eigenvalue decomposition of kernel matrix is intractable. Therefore, using Eq. (\ref{equation_Nystrom_A_eig_decomposition}), we decompose an $m \times m$ submatrix of kernel. 
Comparing Eqs. (\ref{equation_classical_MDS_kernel_inTermsOf_Y}) and (\ref{equation_Nystrom_kernel_OtransposeO}) shows that:
\begin{align}\label{equation_Nystrom_Y}
\mathbb{R}^{n \times n} \ni \b{Y} = [\b{R}, \b{S}] \overset{(a)}{=} [\b{\Sigma}^{(1/2)} \b{U}^\top, \b{\Sigma}^{(-1/2)} \b{U}^\top \b{B}],
\end{align}
where $(a)$ is because of Eqs. (\ref{equation_Nystrom_R}) and (\ref{equation_Nystrom_S}) and the terms $\b{U}$ and $\b{\Sigma}$ are obtained from Eq. (\ref{equation_Nystrom_A_eig_decomposition}). 
The Eq. (\ref{equation_Nystrom_Y}) gives the approximately embedded data, with a good approximation. This is the embedding in \textit{landmark MDS} \cite{de2003global,de2004sparse}. Truncating this matrix to have $\b{Y} \in \mathbb{R}^{p \times n}$, with top $p$ rows, gives the $p$-dimensional embedding of the $n$ points. 

Comparing Eq. (\ref{equation_Nystrom_Y}) with Eq. (\ref{equation_metric_MDS_Y}) shows that the formulae for embedding of landmarks, $\b{R}$, and the whole data (without Nystrom approximation) are similar to each other but one is with only landmarks and the other is with the whole data.  

\begin{figure*}[!t]
\centering
\includegraphics[width=6in]{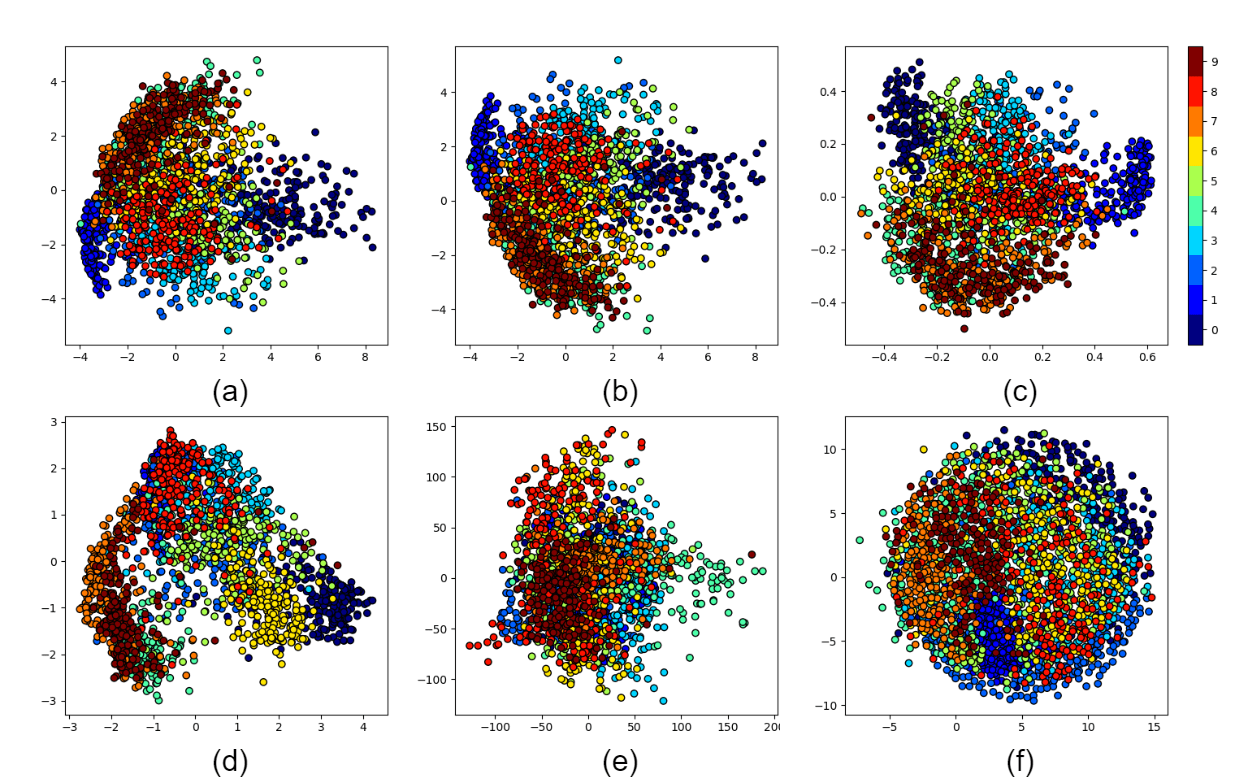}
\caption{Embedding of the training data in (a) classical MDS, (b) PCA, (c) kernel classical MDS (with cosine kernel), (d) Isomap, (e) kernel Isomap, and (f) Sammon mapping.}
\label{figure_simulation_embedding}
\end{figure*}

\begin{figure}[!t]
\centering
\includegraphics[width=2.8in]{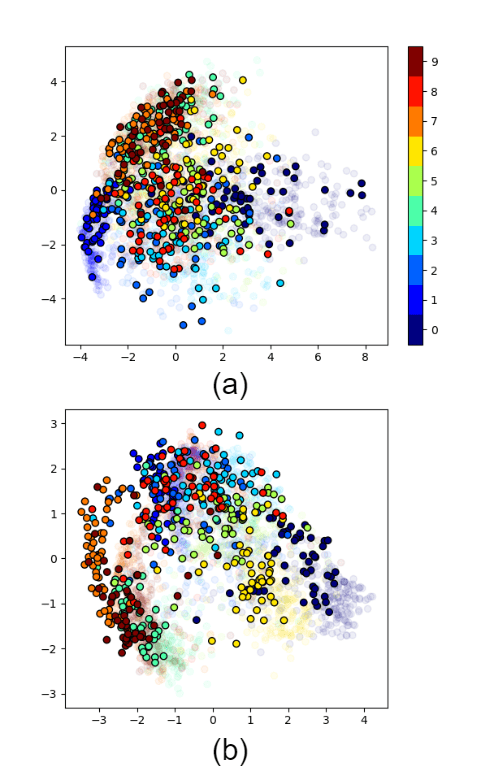}
\caption{Embedding of the out-of-sample data in (a) classical MDS, and (b) Isomap. The transparent points indicate the embedding of training data.}
\label{figure_simulation_embedding_outofsample}
\end{figure}

\subsection{Using Distance Matrix in Landmark MDS}

If $\b{D}_{ij}$ denotes the $(i,j)$-th element of the distance matrix and $\b{v}_j$ is the $j$-th element of a vector $\b{v}$, Eq. (\ref{equation_generalKernel_and_distanceMAtrix}) can be restated as \cite{platt2005fastmap}:
\begin{equation}
\begin{aligned}
\b{K} = \frac{-1}{2} \Big(&\b{D}_{ij}^2 - \b{1}_j \sum_i \b{c}_i \b{D}_{ij}^2 \\
&- \b{1}_i \sum_j \b{c}_j \b{D}_{ij}^2 + \sum_{i,j} \b{c}_i \b{c}_j \b{D}_{ij}^2\Big),
\end{aligned}
\end{equation}
where $\sum_i \b{c}_i = 1$. 

Let the partitions of the distance matrix be:
\begin{align}\label{equation_Nystrom_partions_distance}
\mathbb{R}^{n \times n} \ni \b{D} = 
\left[
\begin{array}{c|c}
\b{E} & \b{F} \\
\hline
\b{F}^\top & \b{G}
\end{array}
\right],
\end{align}
where $\b{E} \in \mathbb{R}^{m \times m}$, $\b{F} \in \mathbb{R}^{m \times (n-m)}$, and $\b{G} \in \mathbb{R}^{(n-m) \times (n-m)}$ in which $m \ll n$.
Comparing Eqs. (\ref{equation_Nystrom_partions}) and (\ref{equation_Nystrom_partions_distance}) shows that the partitions of the kernel matrix can be obtained from the partitions of the distance matrix as \cite{platt2005fastmap}:
\begin{equation}\label{equation_Nystrom_A_from_distance}
\begin{aligned}
\b{A}_{ij} = \frac{-1}{2} \Big(&\b{E}_{ij}^2 - \b{1}_i \frac{1}{m} \sum_p \b{E}_{pj}^2 \\
&- \b{1}_j \frac{1}{m} \sum_q \b{E}_{iq}^2 + \frac{1}{m^2} \sum_{p,q} \b{E}_{pq}^2\Big),
\end{aligned}
\end{equation}
\begin{equation}\label{equation_Nystrom_B_from_distance}
\begin{aligned}
\b{B}_{ij} = \frac{-1}{2} \Big(&\b{F}_{ij}^2 - \b{1}_i \frac{1}{m} \sum_p \b{F}_{qj}^2 - \b{1}_j \frac{1}{m} \sum_q \b{E}_{iq}^2 \Big),
\end{aligned}
\end{equation}
and $\b{C}_{ij}$ can be obtained from Eq. (\ref{equation_Nystrom_C}). 

In landmark MDS and landmark Isomap, the partitions (submatrices) $\b{E}$ and $\b{F}$ of the Euclidean and geodesic distance matrices are calculated, respectively (see Eq. (\ref{equation_Nystrom_partions_distance})).  
Then, Eqs. (\ref{equation_Nystrom_A_from_distance}), (\ref{equation_Nystrom_B_from_distance}), and (\ref{equation_Nystrom_C}) give us the partitions of the kernel matrix. Eqs. (\ref{equation_Nystrom_A_eig_decomposition}) and (\ref{equation_Nystrom_Y}) provide the embedded data.

It is noteworthy that the paper \cite{platt2005fastmap} shows that different landmark MDS methods, such as \textit{Landmark MDS (LMDS)} \cite{de2003global,de2004sparse}, \textit{FastMap} \cite{faloutsos1995fastmap}, and \textit{MetricMap} \cite{wang1999evaluating} are reduced to landmark MDS introduced here. 
The landmark MDS is also referred to as the \textit{sparse MDS} \cite{de2004sparse}.
Moreover, the \textit{Landmark Isomap (L-Isomap)} \cite{de2003global} is reduced to the landmark Isomap method explained here (see {\citep[Corollary 1]{bengio2004out}} for proof). 
In other words, the large-scale manifold learning methods make use of the Nystrom approximation \cite{talwalkar2008large}.

\section{Simulations}\label{section_simulations}

\subsection{Dataset}

For simulations, we used the MNIST dataset \cite{web_mnist_dataset} includes 60,000 training images and 10,000 test images of size $28 \times 28$ pixels. It includes 10 classes for the 10 digits, 0 to 9. 
Because of tractability of the eigenvalue problem, we used a subset of 2000 training points (200 per class) and 500 test points (50 per class). 

\subsection{Training Embedding}

\subsubsection{Classical MDS and Comparison to PCA}

The embedding of training data by classical MDS is shown in Fig. \ref{figure_simulation_embedding}. As can be seen, this embedding is interpretable because, for example, the digits (7 and 9), (6 and 8), and (5 and 6), which can be converted to each other by slight changes, are embedded close to one another. 

Figure \ref{figure_simulation_embedding} also depicts the embedding of training data by PCA. As can be seen, the embedding of PCA is equivalent to the embedding of classical MDS because rotation and flipping does not matter in manifold learning. This validates the claim of equivalence of PCA and classical MDS, stated in Section \ref{section_PCA_MDS_equivalence}. 

\subsubsection{Kernel classical MDS, Isomap, Kernel Isomap, and Sammon Mapping}

The embedding of kernel classical MDS or the generalized classical MDS (with cosine kernel) is also shown in Fig. \ref{figure_simulation_embedding}. This figure also includes the embedding by Isomap. 
It is empirically observed that the embeddings by Isomap are usually like the legs of an octopus \cite{ghojogh2019feature}. In this embedding, you can see two legs one of which is bigger than the other. 
Figure \ref{figure_simulation_embedding} also shows the embedding by kernel Isomap. Note that kernel Isomap still uses the kernel calculated using the geodesic distance. 
Finally, the embedding by Sammon mapping, with 1000 iterations, is also illustrated in Fig. \ref{figure_simulation_embedding}. The embeddings by all these methods are meaningful because the more similar digits have been embedded close to each other. 

An important fact about the embeddings is that the mean is zero in the embeddigns by classical MDS, kernel classical MDS, Isomap, and kernel Isomap. This is because of double centering the distance matrices in these methods (see Eqs. (\ref{equation_linearKernel_and_distanceMAtrix_2}), (\ref{equation_generalKernel_and_distanceMAtrix}), (\ref{equation_Kernel_Isomap}), and (\ref{equation_generalKernel_and_distanceMAtrix_squared})). 

% \subsubsection{Sammon Mapping}

% \subsubsection{Isomap}

% \subsubsection{Kernel Isomap}

\subsection{Out-of-sample Embedding}

% \subsubsection{Classical MDS}

The out-of-sample embedding of the classical MDS and Isomap can be seen in Fig. \ref{figure_simulation_embedding_outofsample}. For the out-of-sample embeddings by classical MDS and Isomap, we used Eqs. (\ref{equation_embedding_eigenfunction}) and (\ref{equation_Isomap_outOfSample}), respectively. In the Isomap method, as it is difficult to implement the geodesic distance matrix calculated from only the training points as the intermediate points, we used an approximation in which the test points can also be used as intermediate points. A slight shift in the mean of out-of-sample embedding in the Isomap result is because of this approximation. 

% \subsubsection{Isomap}

\subsection{Code Implementations}

The Python code implementations of simulations can be found in the repositories of the following github profile: \url{https://github.com/bghojogh}

\section{Conclusion}\label{section_conclusion}

This tutorial and survey paper was on MDS, Sammon mapping, and Isomap. Classical MDS, kernel classical MDS, metric MDS, and non-metric MDS were explained as categories of MDS. Sammon mapping and Isomap were also explained as special cases of metric MDS and kernel classical MDS. Kernel Isomap was also introduced. Out-of-sample extensions of these methods using eigenfunctions and kernel mapping were also provided. Landmark MDS and landmark Isomap using Nystron approximation were also covered in this paper. Finally, some simulations were provided to show the embeddings. 

Some specific methods, based on MDS and Isomap were not covered in this paper for the sake of brevity. some examples of these methods are supervised Isomap \cite{wu2004extended}, robust kernel Isomap  \cite{choi2007robust} and kernel Isomap for noisy data \cite{choi2005kernel}. 

\bibliography{References}
\bibliographystyle{icml2016}

\end{document}